\documentclass{article}

 \usepackage[final]{bdl_2018}

\usepackage[utf8]{inputenc} 
\usepackage[T1]{fontenc}    
\usepackage{hyperref}       
\usepackage{url}            
\usepackage{booktabs}       
\usepackage{amsfonts}       
\usepackage{nicefrac}       
\usepackage{microtype}      
\usepackage{amsthm}
\newtheorem{theorem}{Theorem}

\usepackage{amsmath,amsfonts,amssymb}

\usepackage[caption=false]{subfig}
\usepackage{graphicx}
\title{Stochastic natural gradient descent draws \\ posterior samples in function space}

\author{Samuel L. Smith*$^-$, Daniel Duckworth*, Semon Rezchikov$^{+\pm}$, Quoc V. Le* and Jascha Sohl-Dickstein* \\
  *Google Brain, $^-$DeepMind, $^+$Google, $^\pm$Columbia University\\
  \texttt{\{slsmith, duckworthd, qvl, jaschasd\}@google.com, skr@math.columbia.edu} \\
}

\begin{document}

\maketitle

\begin{abstract}
Recent work has argued that stochastic gradient descent can approximate the Bayesian uncertainty in model parameters near local minima. In this work we develop a similar correspondence for minibatch natural gradient descent (NGD). We prove that for sufficiently small learning rates, if the model predictions on the training set approach the true conditional distribution of labels given inputs, the stationary distribution of minibatch NGD approaches a Bayesian posterior near local minima. The temperature $T = \epsilon N / (2B)$ is controlled by the learning rate $\epsilon$, training set size $N$ and batch size $B$. However minibatch NGD is not parameterisation invariant and it does not sample a valid posterior away from local minima. We therefore propose a novel optimiser, ``stochastic NGD'', which introduces the additional correction terms required to preserve both properties.
\end{abstract}

\section{Introduction}

Stochastic gradient descent remains the dominant optimization algorithm for supervised learning, but it performs poorly when the loss landscape is ill-conditioned. Natural gradient descent (NGD) provides a robust alternative based on the functional view of learning \citep{amari1998natural}. This view observes that the model parameters $\omega$ are not important, only the function $f_{\omega}$ they represent. This function relates randomly sampled inputs $X$ to a distribution over labels $Y \sim f_{\omega}(X)$, and the goal of learning is to uncover the true conditional distribution $P(Y|X)$. We will use $\{x,y\}$ to refer to the training set, composed of individual training examples $\{x_i,y_i\}$. A single NGD update has the form,
\begin{equation}
\Delta \omega = \epsilon F_x(\omega)^{-1} \left( \frac{1}{N} \frac{dC}{d\omega} \right) \label{eqn:intro},
\end{equation}
where $C_{x,y}(\omega)$ is the summed cost function, $N$ is the training set size, and $\Delta \omega = (\omega_{t+1} - \omega_{t})$ is the parameter update. We pre-condition the gradient by the inverse Fisher information, $F_x(\omega)^{-1}$ (defined in Section 4), which we estimate over the training inputs \citep{ly2017tutorial}. As the learning rate $\epsilon \rightarrow 0$, the change in function $(f_{\omega_{t+1}}(X) - f_{\omega_{t}}(X))$ during an update becomes independent of the model parameterization \citep{martens2014new}. NGD also ensures stability, by bounding the KL-divergence between functions before and after an update, such that $D_{KL}(f_{\omega_t}(X) \mid f_{\omega_{t+1}}(X)) \leq \epsilon^2$, subject to weak requirements on the smoothness and support of $f_{\omega}(X)$ \citep{ollivier2011information}.

However these properties require full-batch gradients, while in practice we estimate the gradient over a minibatch \citep{bottou2010large, robbins1951stochastic}. Therefore in this work we analyse the influence of stochastic gradients on the stationary distribution of NGD as $\epsilon \rightarrow 0$\footnote{Note however that throughout this work we assume the uncertainty in the estimate of the Fisher is negligible.}. Remarkably, we find that if the model predictions on the training set $f_{\omega}(x_i)$ approach the true conditional distribution $P(Y|x_i)$, then this stationary distribution approaches a Bayesian posterior near local minima. The temperature $T = \frac{\epsilon N}{2B}\left(1-\frac{B}{N}\right) \approx \frac{\epsilon N}{2B}$ is controlled by the ratio of the learning rate, $\epsilon$, to the batch size, $B$.

However unlike full-batch NGD, minibatch NGD is not parameterisation invariant. To recover parameterisation invarance, we introduce additional terms which arise from the parameter dependence of the Fisher metric. We propose a novel optimiser, ``stochastic natural gradient descent" (SNGD),

\begin{eqnarray}
&&\Delta \omega_i = - \epsilon F_x(\omega)^{-1}  \left( \frac{1}{N} \frac{d\hat{C}}{d\omega}  \right)_i    \label{eqn:sngd} \\
                    && - \frac{\epsilon^2}{2B}\left(1- \frac{B}{N}\right) \left\{\sum_j \left( F_x(\omega)^{-1} \frac{dF_x}{d\omega_j} F_x(\omega)^{-1} \right)_{ij}- \frac{1}{2} \sum_j F_x(\omega)^{-1}_{ij} \, \mathrm{Tr} \left(F_x(\omega)^{-1} \frac{dF_x}{d\omega_j}\right) \right\}  \nonumber .
\end{eqnarray}

SNGD reduces to a conventional NGD step as $(\epsilon/B) \rightarrow 0$ or when $B \rightarrow N$. Meanwhile in the limit $\epsilon \rightarrow 0$ and $f_{\omega}(x_i) \rightarrow P(Y|x_i)$, Equation \ref{eqn:sngd} will preserve parameterisation invariance and draw samples from a valid Bayesian posterior throughout parameter space at temperature $T \approx \epsilon N/ (2B)$. This novel optimiser introduces a multiplicative bias to the stationary distribution, which we identify as the Jeffreys prior \citep{jeffreys1946invariant}. This popular default prior is both uninformative and invariant to the model parameterization \citep{firth1993bias, gelman2009bayes}, enabling us to sample from ``function space'', rather than parameter space. We provide a range of experiments to support our claims:

\begin{itemize}
\item We show that when the batch size $B \approx \epsilon N/2$ and $N \rightarrow \infty$, the distribution of samples from NGD is close to the Laplace approximation to the true Bayesian posterior at $T=1$.

 \item If the prior over parameters is chosen appropriately, then Bayesians recommend drawing inferences using ensembles sampled at $T=1$. We observe empirically that the test loss of ensembles drawn using NGD exhibit minima close to $B \approx \epsilon N /2$. This suggests we can analytically predict the optimal ratio between the learning rate and the batch size.
\end{itemize}

Our work builds on a number of recent works in the Bayesian community, which we discuss in Section 2. In Section 3 we derive the update rules for Langevin posterior sampling on arbitrary Riemannian manifolds. In Section 4 we consider Langevin dynamics on the Fisher manifold, noting that the parameter dependence of the Fisher information introduces a data-dependent Jeffreys prior over the parameters. In section 5 we analyse the behaviour of minibatch NGD near local minima, and we introduce our novel stochastic NGD which preserves parameterisation invariance throughout parameter space. We present our experimental results in Section 6, applying minibatch NGD to logistic regression, shallow fully connected neural networks and deep convolutional networks.
\section{Related work}

Langevin dynamics \citep{gardiner1985handbook} can be used to draw posterior samples but often mixes slowly. \citet{girolami2011riemann} proposed Riemannian Langevin dynamics, which combines Langevin dynamics and NGD to enable faster mixing when the loss landscape is poorly conditioned. However their method still requires exact gradients and remains prohibitively expensive for large training sets. \citet{welling2011bayesian} proposed SGLD, the first algorithm to draw posterior samples with stochastic gradients. \citet{patterson2013stochastic} combined SGLD with Riemannian Langevin dynamics while \citet{li2016preconditioned} combined SGLD with RMSprop. These methods scale well to large datasets but they still mix slowly, since the learning rate must satisfy the Robbins-Munro conditions \citep{robbins1951stochastic}. \citet{ahn2012bayesian} observed that the noise inherent in stochastic gradients converges to the empirical Fisher near local minima. By explicitly accounting for this noise, they draw samples from a Gaussian approximation to the posterior with finite learning rates, while converging to the exact posterior as $\epsilon \rightarrow 0$. A similar analysis was published by \citet{mandt2017stochastic}, who set the gradient preconditioner equal to the inverse of the empirical gradient covariances. These methods enable fast mixing with stochastic gradients, but the analysis in both papers only holds near local minima.

The link between natural gradients and the Jeffreys prior is mentioned in earlier work by \citet{amari1998natural}, while \citet{mahony2001prior} analysed the relationship between Riemannian manifolds and implicit priors in online learning. While we focus on exact NGD in this work, we anticipate that our conclusions could be scaled to larger networks using approximate natural gradient methods like K-FAC \citep{martens2015optimizing} or other quasi-diagonal approximations to the Fisher \citep{marceau2017natural}. \citet{zhang2017noisy} recently proposed applying K-FAC to perform variational inference, while \citet{nado2018stochastic} combined K-FAC with SGLD for posterior sampling.

\section{Posterior sampling on Riemannian manifolds}
The Langevin equation updates the parameters in continuous time $t$ according to \citep{gardiner1985handbook},
\begin{equation}
\frac{d\omega}{dt} = - \frac{dC}{d\omega}  + \eta(t) \label{eqn:post},
\end{equation}
where $C_{x,y}(\omega)$ is the summed cost function and $\eta(t)$ denotes isotropic Gaussian noise. This noise has mean $\mathbb{E}\left(\eta(t)\right) = 0$ and covariance $\mathbb{E}\left(\eta(t) \eta(t')^\top\right) = 2T I  \delta(t-t')$, where $\delta(t-t')$ is a Dirac delta function, $I$ is the identity matrix, and $T$ is a scalar known as the ``temperature''. In the infinite time limit, Equation \ref{eqn:post} will sample parameters from $P_{t \rightarrow \infty}(\omega) \propto e^{-C_{x,y}(\omega)/T}$. To draw samples numerically we introduce the learning rate $\epsilon$, and integrate over a finite step of length $\epsilon/N$ to obtain,
\begin{equation}
 \Delta \omega  = - \epsilon \left( \frac{1}{N} \frac{dC}{d\omega} + \alpha \right), \label{eqn:post_discrete}
\end{equation}
where $\alpha$ is a Gaussian random variable with mean $\mathbb{E}\left(\alpha\right) = 0$ and covariance $\mathbb{E}\left(\alpha \alpha^\top\right) = 2 T I/(\epsilon N)$. If we perform $r$ parameter updates using Equation \ref{eqn:post_discrete} and $\epsilon \rightarrow 0$, then we will draw a single sample from the distribution above as $t = r\epsilon/N \rightarrow \infty$. Meanwhile if the posterior is proportional to the exponentiated negative cost function, $P(\omega|x,y) \propto e^{-C_{x,y}(\omega)}$, then we sample parameters from this posterior when $T=1$. We now generalise the Langevin equation to a static Riemannian manifold,

\begin{equation}
\Delta \omega = - \epsilon G_{x,y}^{-1}  \left( \frac{1}{N} \frac{dC}{d\omega} + \alpha \right), \label{eqn:static}
\end{equation}

\begin{theorem}
We assume that $C_{x,y}(\omega)$ is Lipschitz continuous everywhere, and that the metric $G_{x,y}$ is positive definite. In the limit $\epsilon \rightarrow 0$, repeated application of Equation \ref{eqn:static} will sample parameters from the stationary distribution $P_{t\rightarrow \infty}(\omega) \propto e^{-C_{x,y}(\omega)/T}$, if $\mathbb{E}(\alpha) = 0$ and $\mathbb{E}(\alpha \alpha^\top) = 2 T G_{x,y} /(\epsilon N)$.
\end{theorem}

\begin{proof} Since $G_{x,y}$ is positive definite, $G_{x,y} = B_{x,y} B_{x,y}^\top$. To prove Theorem 1, we simply perform a basis transformation on the original Langevin equation of Equation \ref{eqn:post}, such that $\omega = B_{x,y}^\top \omega'$.  \end{proof}

Note however that the metric $G_{x,y}$ in Equation \ref{eqn:static} is constant and independent of the parameters. Meanwhile the metric of NGD is defined by the Fisher information, which is a function the parameters. A detailed discussion is beyond the scope of this paper, but great care must be taken when discretising stochastic differential equations on a non-stationary metric. To build an intuition for this, we note that the magnitude of the noise contribution in Equation \ref{eqn:static} is $O(\sqrt{\epsilon})$. We therefore cannot expand a stochastic differential equation to first order in $\epsilon$ without considering second order contributions from the noise source, and these second order contributions introduce additional correction terms which vanish if the metric derivatives are zero. We introduce these correction terms below,
\begin{eqnarray}
\Delta \omega_i &=& - \epsilon \Bigg\{ G_{x,y}(\omega)^{-1}  \left( \frac{1}{N} \frac{dC}{d\omega} + \alpha  \right)_i +  \label{eqn:massive} \\
                    && \frac{T}{N} \sum_j \left( G_{x,y}(\omega)^{-1} \frac{dG_{x,y}}{d\omega_j} G_{x,y}(\omega)^{-1} \right)_{ij}- \frac{T}{2N} \sum_j G_{x,y}(\omega)^{-1}_{ij} \, \mathrm{Tr} \left(G_{x,y}(\omega)^{-1} \frac{dG_{x,y}}{d\omega_j}\right) \Bigg\}  \nonumber ,
\end{eqnarray}

\begin{theorem}
We assume both $C_{x,y}(\omega)$ and $G_{x,y}(\omega)$ are Lipschitz continuous and $G_{x,y}(\omega)$ is positive definite. In the limit $\epsilon \rightarrow 0$, repeated application of Equation \ref{eqn:massive} will sample parameters from $P_{t\rightarrow \infty}(\omega) \propto e^{-C_{x,y}(\omega)/T} |\det{G_{x,y}(\omega)}|^{1/2}$, if $\mathbb{E}(\alpha) = 0$ and $\mathbb{E}(\alpha \alpha^\top) = 2 T G_{x,y}(\omega) /(\epsilon N)$.
\end{theorem}

This theorem is discussed further in the appendix. Notice that the stationary distribution is modified by an implicit bias, $|\det{G_{x,y}(\omega)}|^{1/2}$, which arises from the parameter dependence of the metric. If one does not wish to introduce this implicit bias, one could instead apply the following update rule,
\begin{eqnarray}
\Delta \omega_i &=& - \epsilon \Bigg\{ G_{x,y}(\omega)^{-1}  \left( \frac{1}{N} \frac{dC}{d\omega} +\alpha \right)_i + \frac{T}{N} \sum_j \left( G_{x,y}(\omega)^{-1} \frac{dG_{x,y}}{d\omega_j} G_{x,y}(\omega)^{-1} \right)_{ij} \Bigg\}  . \label{eqn:flat}
\end{eqnarray}

\begin{theorem}
We assume both $C_{x,y}(\omega)$ and $G_{x,y}(\omega)$ are Lipschitz continuous and $G_{x,y}(\omega)$ is positive definite. In the limit $\epsilon \rightarrow 0$, repeated application of Equation \ref{eqn:flat} will sample parameters from $P_{t\rightarrow \infty}(\omega) \propto e^{-C_{x,y}(\omega)/T} $, if $\mathbb{E}(\alpha) = 0$ and $\mathbb{E}(\alpha \alpha^\top) = 2 T G_{x,y}(\omega) /(\epsilon N)$.
\end{theorem}
\begin{proof}
Consider Theorem 2, and let $C_{x,y}(\omega) \rightarrow C'_{x,y}(\omega)+ (T/2) \ln{|\det{G_{x,y}|}}$. This modified cost function removes the implicit bias from the stationary distribution, such that $P_{t\rightarrow \infty}(\omega) \propto e^{-C'_{x,y}(\omega)} $, and by directly evaluating the derivative of $\ln{|\det{G_{x,y}}|}$ we arrive at Equation \ref{eqn:flat}.
\end{proof}

\section{Preconditioned Langevin dynamics and the Jeffreys prior}
In this Section, we apply Theorem 2 to draw posterior samples with full batch gradient estimates on the Fisher manifold. For simplicity, we assume the cost function $C_{x,y}(\omega) = \sum_i \mathcal{L}(\omega, x_i, y_i) + \lambda R_{x}(\omega)$, where $R_{x}(\omega)$ is a regularizer and $\mathcal{L}(\omega,x_i,y_i) = - \ln P(y_i|x_i,\omega)$ is the cross-entropy of a unique categorical label. We set the metric $G_{x,y}(\omega)$ equal to the Fisher information,
\begin{eqnarray}
F_x(\omega) &=& \frac{1}{N} \sum_{i=1}^N \mathbb{E}_{Y \sim f_{\omega}(x_i)}\left( \frac{d \mathcal{L}(\omega, x_i, Y)}{d\omega} \frac{d \mathcal{L}(\omega,x_i,Y)}{d\omega^\top} \right) .
\end{eqnarray}
Notice that the expectation is taken over the predicted label distribution given the current parameter values, $f_{\omega}(x_i)$, not the empirical labels $y_i$, and therefore $F_x(\omega)$ depends only on the training inputs $x$. Following Theorem 2, for sufficiently small $\epsilon$ repeated application of Equation \ref{eqn:massive} will draw samples from $P_{t\rightarrow \infty}(\omega) \propto e^{-C_{x,y}(\omega)/T} |F_{x}(\omega)|^{1/2}$. The data likelihood $P(y|x,\omega) = e^{-\sum_i \mathcal{L}(\omega,x_i,y_i)}$. We introduce the temperature adjusted posterior $P_T(\omega|x,y) \propto P(y|x,\omega)^{1/T} P_T(\omega|x)$, and recall that we obtain the true posterior $P(\omega|x,y) = P_1(\omega|x,y)$ when $T=1$. Notice that the temperature adjusted prior $P_T(\omega|x)$ may depend on the training inputs $x$ but not the labels $y$. We set $P_{t\rightarrow \infty}(\omega) = P_T(\omega| x, y)$ to identify,
\begin{equation}
P_T(\omega|x) \propto e^{- \lambda R_x(\omega)/T} |F_{x}(\omega)|^{1/2}.
\end{equation}
The multiplicative bias introduced by the Fisher metric modifies the prior imposed by the regulariser. When $\lambda \rightarrow 0$, we obtain a temperature independent prior $P(\omega|x) \propto |F_{x}(\omega)|^{1/2}$, which we identify as the Jeffreys prior; a common default prior which is both uninformative and invariant to the model parameterization \citep{jeffreys1946invariant, firth1993bias}. Like the uniform prior, the Jeffreys prior is improper. It may therefore be necessary to include additional regularisation to ensure that the posterior is well-defined. Notice that the contribution to the prior from the regulariser grows weaker as the temperature increases, matching the temperature dependence of the likelihood term. Consequently the contribution to the prior from the metric becomes increasingly dominant as the temperature rises.

Note that we can only reinterpret the multiplicative bias as a prior if the gradient pre-conditioner is independent of the training labels. However practitioners often replace the Fisher information by the empirical Fisher information, estimating the inner expectation above over the training set labels $y$,
\begin{equation}
F_{x,y}(\omega) = \frac{1}{N} \sum_{i=1}^N \left( \frac{d \mathcal{L}(\omega,x_i,y_i)}{d\omega} \frac{d \mathcal{L}(\omega,x_i,y_i)}{d\omega^\top} \right) .
\end{equation}
In this case, Equation \ref{eqn:massive} will draw samples from $e^{-C_{x,y}(\omega)/T} |\det{F_{x,y}}|^{1/2}$ but we can only interpret this as a Bayesian posterior if the empirical Fisher is constant (for instance near local minima), in which case the correction terms vanish and we recover the simpler update rule of Equation \ref{eqn:static}. We could draw valid posterior samples under a uniform prior with the empirical fisher and Equation \ref{eqn:flat}.

\section{Stochastic natural gradients and minibatch noise}
We showed in Section 4 how to draw posterior samples with full batch natural gradients and Gaussian noise. This required us to introduce additional correction terms, as shown in Equation \ref{eqn:massive}. However in practice we usually estimate the gradient over a minibatch. The minibatch NGD update has the form,
\begin{eqnarray}
\Delta \omega  = - \epsilon F_{x}(\omega)^{-1} \left( \frac{1}{N} \frac{dC}{d\omega} + \beta \right), \label{eqn:sngd_step}
\end{eqnarray}
where the gradient noise,
\begin{eqnarray}
\beta &=& \frac{1}{B} \left( \sum_{i=1}^B \frac{d\mathcal{L}(\omega,x_i,y_i)}{d\omega} \right) - \frac{1}{N} \left( \sum_{i=1}^N \frac{d\mathcal{L}(\omega,x_i,y_i)}{d\omega} \right) \label{eqn:minibatch}.
\end{eqnarray}
We assume the training set is randomly sorted between each update, such that Equation \ref{eqn:minibatch} samples a minibatch of $B$ training examples without replacement. Below we prove Theorem 4,

\begin{theorem}
We assume $C_{x,y}(\omega)$ is Lipschitz continuous everywhere, and that the Fisher information matrix $F_x(\omega)$ is positive definite. Training inputs are drawn from a fixed distribution $P(X)$, while labels are assigned by a conditional distribution $P(Y|X)$. If $\epsilon \rightarrow 0$, $1 \ll B \ll N$, and $f_{\omega}(x_i) \rightarrow P(Y|x_i)$ for all observed $x_i \in x$, repeated application of Equation \ref{eqn:sngd_step} will draw samples from $P_{t\rightarrow \infty}(\omega) \propto e^{-C_{x,y}(\omega)/T}$ near local minima (where the Fisher is stationary), with $T = \frac{\epsilon N}{2B}(1- \frac{B}{N})$.
\end{theorem}
\begin{proof}
Since the gradients of individual training examples are independent,
\begin{equation}
\mathbb{E}\left( \frac{d \mathcal{L}(\omega,x_i,y_i)}{d\omega} \frac{d \mathcal{L}(\omega,x_j,y_j)}{d\omega^\top} \right) = \frac{1}{N^2}  \left( \sum_{k=1}^N \frac{d\mathcal{L}(\omega,x_k,y_k)}{d\omega} \right) \left( \sum_{k=1}^N \frac{d\mathcal{L}(\omega,x_k,y_k)}{d\omega^\top} \right) + \Sigma_{x,y}(\omega) \delta_{ij} \label{eqn:long}.
\end{equation}
$\Sigma_{x,y}(\omega)$ denotes the empirical gradient covariances. Applying the central limit theorem over examples in the minibatch, we conclude the gradient noise $\beta$ is a Gaussian random variable. By direct substitution $\mathbb{E}( \beta ) = 0$, while the covariance $\mathbb{E}( \beta \beta^\top )= \frac{N}{B}(1- \frac{B}{N}) \Sigma_{x,y}(\omega)$. To compute $\Sigma_{x,y}(\omega)$, we sum Equation \ref{eqn:long} over the indices $(i,j)$,
\begin{eqnarray}
\Sigma_{x,y}(\omega) &=& F_{x,y}(\omega) - \frac{1}{N^2}  \left( \sum_{k=1}^N \frac{d\mathcal{L}(\omega,x_k,y_k)}{d\omega} \right) \left( \sum_{k=1}^N \frac{d\mathcal{L}(\omega,x_k,y_k)}{d\omega^\top} \right).
\end{eqnarray}
As $N\rightarrow \infty$, $\Sigma_{x,y}(\omega) \rightarrow \Sigma(\omega)$, the covariance of the underlying data distribution. Meanwhile under mild regularity conditions, $\mathbb{E}_{Y\sim f_{\omega}(x_i)}\left( \frac{d\mathcal{L}(\omega,x_i,Y)}{d\omega} \right) = 0$, and consequently if $f_{\omega}(x_i) \rightarrow P(Y|x_i)$ for all $x_i \in x$, we also obtain $F_x(\omega) \rightarrow \Sigma(\omega)$. Since we assumed $F_x(\omega) \succ 0$ this implies $\Sigma(\omega) \succ 0$, which ensures the gradient noise cannot vanish and prevents the stationary distribution from collapsing to a fixed point. Finally we note that near local minima, the Fisher metric $F_{x}(\omega) \rightarrow F_x(\omega_0)$ is stationary. Comparing Equations \ref{eqn:static} and \ref{eqn:sngd_step} and following Theorem 1, we obtain Theorem 4.
\end{proof}
This remarkable result shows that as the predictions of the model on the training set grow closer to the true conditional distribution of labels given inputs, the stationary distribution of minibatch NGD approaches a Bayesian posterior near local minima at temperature $T  \approx \epsilon N/(2B)$ (at least as $\epsilon \rightarrow 0$). A similar analysis was proposed by \citet{ahn2012bayesian}, while \cite{mandt2017stochastic} proposed to directly precondition the gradient by the empirical covariances. Both these works also require that the Fisher is stationary. In order to extend our analysis to non-stationary Fisher matrices, we replace the conventional minibatch NGD step of Equation \ref{eqn:sngd_step} by our novel SNGD step proposed in Equation \ref{eqn:sngd}.

\begin{theorem}
We assume both $C_{x,y}(\omega)$ and $F_{x}(\omega)$ are Lipschitz continuous everywhere, and that the Fisher information matrix $F_x(\omega)$ is positive definite. Training inputs are drawn from a fixed distribution $P(X)$, while labels are assigned by a conditional distribution $P(Y|X)$. If $\epsilon \rightarrow 0$, $1 \ll B \ll N$, and $f_{\omega}(x_i) \rightarrow P(Y|x_i)$ for all observed $x_i \in x$, repeated application of Equation \ref{eqn:sngd} will draw samples from $P_{t\rightarrow \infty}(\omega) \propto e^{-C_{x,y}(\omega)/T} |\det{F_x(\omega)}|^{1/2}$, where $T = \frac{\epsilon N}{2B}(1 -\frac{B}{N})$.
\end{theorem}
\begin{proof}
This result follows directly from the proofs of Theorems 2 and 4. Notice that Equation \ref{eqn:massive} reduces to Equation \ref{eqn:sngd} if one sets $G_{x,y}(\omega) \rightarrow F_x(\omega)$ and $T \rightarrow \frac{\epsilon N}{2B} (1 - \frac{B}{N})$. 
\end{proof}

As discussed in Section 4, the implicit bias $|\det{F_x(\omega)}|^{1/2}$ in the stationary distribution of Equation \ref{eqn:sngd} introduces a parameterisation invariant Jeffreys prior over the parameters. Notice that we could also exploit Equation \ref{eqn:flat} and Theorem 3 to draw (approximate) posterior samples under a uniform prior.

Finally, we note that the metric must be positive definite, while the Fisher is positive semi-definite. To resolve this, it is common practice to introduce the modified metric $G_x(\omega) = F_x(\omega) + \delta I$ \citep{martens2014new}. The ``Fisher damping'' $\delta$ imposes a trust region, ensuring the eigenvalues of the pre-conditioner $\epsilon (F_x(\omega) + \delta I)^{-1}$ are bounded by $\epsilon/\delta$. We therefore expect that we will be able to achieve stable training at larger learning rates by increasing $\delta$, while we converge to exact SNGD as $\delta \rightarrow 0$. The Fisher damping will also modify the implicit prior and break parameterisation invariance.
\section{Experiments}

In the following, we provide experiments analysing the behaviour of conventional minibatch NGD in the light of Theorem 4. We leave empirical analysis of our novel ``SNGD'' update rule to future work.

\subsection{Comparing preconditioned Langevin dynamics and stochastic NGD}

\begin{figure}[t]
    \centering
        \subfloat[]{\includegraphics[width=0.4\columnwidth]{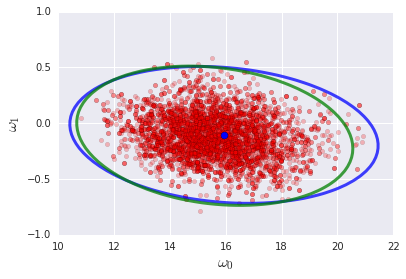}}
        \qquad
        \subfloat[]{\includegraphics[width=0.4\columnwidth]{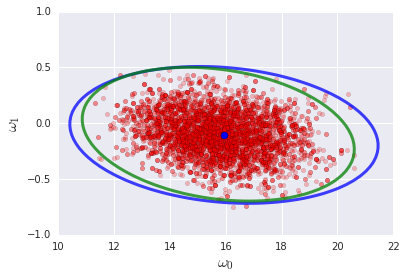}}
    \caption{Samples from (a) preconditioned Langevin dynamics at $T=1$ and (b) NGD when the batch size $B = \epsilon N/2$. We plot the iterates in red, the Laplace approximation to the posterior in blue and the covariance of the iterates in green (to 3 standard deviations). As predicted, the stationary distributions of preconditioned Langevin dynamics and stochastic NGD are remarkably similar.}
\end{figure}

We first consider a binary logistic regression task, with exclusive labels $y_i \in \{-1,1\}$ and inputs $x_i$ drawn from an n-dimensional unit Gaussian. Our model infers labels via $P(y_i| x_i, \omega) = 1/(1+ e^{y_i \omega^\top x})$. In Figure 1a, we show the distribution of samples from preconditioned Langevin dynamics, for which we set the metric $G_{x,y}(\omega)$ equal to the Fisher information $F_{x}(\omega)$, and we evaluate the gradient over the entire training set. To set the temperature, we explicitly add Gaussian noise $\alpha$ to the gradient of variance $\mathbb{E}(\alpha \alpha^\top) = 2TF_x(\omega)/(\epsilon N)$, as described in Section 4\footnote{However, for simplicity we do not include the additional correction terms in Equation \ref{eqn:massive}, instead following the simpler update rule of Equation \ref{eqn:static}, despite the parameter dependence of the Fisher metric.}. Our training set comprises 1000 examples and we set the input dimensionality $n=2$. We sample the training set labels of each input from the Bernoulli distribution of the ``true'' parameter values $\omega_{true} = (16,0)$. The learning rate $\epsilon = 0.1$ and the Fisher damping $\delta = 10^{-4}$. We first run 5000 parameter updates to burn in, before sampling the following 5000 parameter values. The blue point denotes the cost function minimum, obtained by LBFGS, while the blue curve illustrates the width of the Bayesian posterior (to 3 standard deviations), estimated using the Laplace approximation. To confirm that the samples are drawn from a distribution close to the posterior, we plot the covariance of the samples in green. In Figure 1b we replace preconditioned Langevin dynamics by NGD. To set $T=1$, we simply estimate the gradient over minibatches of size $B = \epsilon N/2 = 50$. Batches are sampled randomly without replacement, and we continue to estimate the Fisher information over the full training set. As predicted, when we set the batch size correctly the stationary distribution of NGD is close to the posterior, drawing samples from the same distribution as preconditioned Langevin dynamics.

We now increase the difficulty of the task by setting the input dimensionality $n= 64$ and the true parameter values $w_{true} = (16,0,0,0,...,0)$, such that the input comprises one relevant feature and 63 irrelevant features. In Figure 2, we plot the mean test set accuracy and test cross entropy of preconditioned Langevin dynamics as we vary the temperature $T$. Once again, we set the learning rate $\epsilon = 0.1$ and the damping $\delta = 10^{-4}$, and we run 1000 parameter updates to burn in. In blue we plot the performance of a conventional ensemble average over the final 1000 parameter values, while in green we plot the accuracy of a single sample. The ensemble outperforms single parameter samples, and the test cross entropy of the ensemble is minimised when $T=1$. Meanwhile in Figure 3, we plot the test set accuracy and test cross entropy of NGD on the same task. We continue to estimate the Fisher over the full training set, but we sample the gradient using minibatches of size $B = \epsilon N/(2T)$. As predicted by our theoretical analysis, the performance of NGD is remarkably similar to preconditioned Langevin dynamics, and the test cross entropy is minimised at $T=\epsilon N/(2B) = 1$.

\begin{figure}[t]
    \centering
        \subfloat[]{\includegraphics[width=0.4\columnwidth]{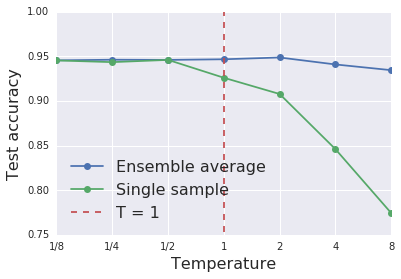}}
        \qquad
        \subfloat[]{\includegraphics[width=0.4\columnwidth]{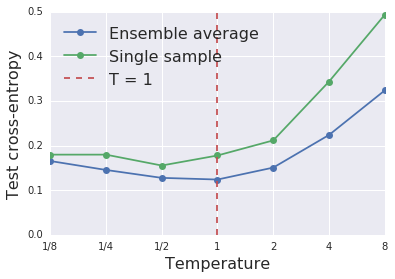}}
    \caption{Preconditioned Langevin dynamics for logistic regression. The test set accuracy (a) and test cross-entropy (b), as a function of the sampling temperature $T$, which we set by adding Gaussian noise to the gradient. The test cross-entropy of the ensemble is minimised at $T=1$.}
\end{figure}

\begin{figure}[t]
    \centering
        \subfloat[]{\includegraphics[width=0.4\columnwidth]{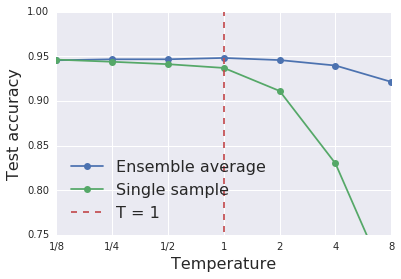}}
        \qquad
        \subfloat[]{\includegraphics[width=0.4\columnwidth]{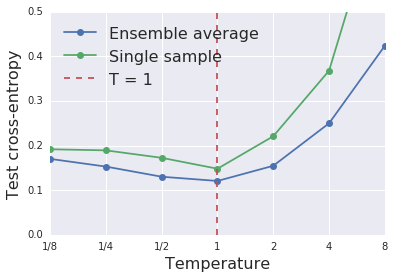}}
    \caption{NGD for logistic regression. The test set accuracy (a) and test cross-entropy (b), as a function of the sampling temperature $T$, which we control by setting the batch size $B = \epsilon N/(2T)$. As predicted, the performance of the ensemble is similar to preconditioned Langevin dynamics.}
\end{figure}

\subsection{Stochastic NGD and an MLP}

To confirm that our conclusions are relevant to non-convex settings we now apply NGD to train a simple MLP on MNIST. Our model has a single hidden layer with 40 hidden units and RELU activations. To reduce the number of parameters, we use a matrix whose elements are drawn randomly from the unit Gaussian to project the input features down to 10 dimensions, and to emphasize the influence of the temperature on training we reduce the training set size to $N = 1024$. We also found it was necessary to include additional regularisation to ensure the samples converge to a stationary distribution. We therefore introduce L2-regularization with regularization coefficient $\lambda = 20.0/N$. We increase the Fisher damping to $\delta = 0.1$ to ensure stability, while the learning rate $\epsilon = 1/8$. In Figure 4a, we plot the test set accuracy of NGD as a function of the batch size used to estimate the gradient. The Fisher is estimated over separate batches of $1024$ images, using a single sampled label per example. For each batch size, we perform 500 gradient updates to burn in, before sampling the parameters over a further 500 updates. We plot both the performance of the ensemble average across the 500 samples, as well as the mean accuracy of a single sample. When training with an ensemble, the accuracy increases rapidly as we increase the batch size until we reach $B = \epsilon N/2 = 64$, at which the temperature $T=1$, while above this temperature the accuracy drops\footnote{We note that the optimum temperature may depend on the damping coefficient or the regulariser, although we would usually expect this dependence to be relatively weak.}. In Figure 4b, we exhibit the test set cross-entropy. While the mean test cross-entropy of a single sample always rises as we reduce the batch size, the ensemble test cross-entropy also shows a minimum at $B = 64$. 

\begin{figure}[t]
    \centering
        \subfloat[]{\includegraphics[width=0.4\columnwidth]{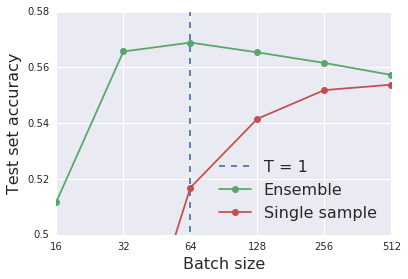}}
        \qquad
        \subfloat[]{\includegraphics[width=0.4\columnwidth]{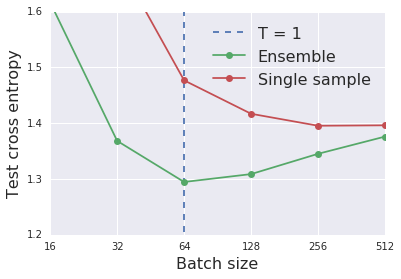}}
    \caption{The test set accuracy (a) and test cross entropy (b) of NGD as a function of batch size $B$, when training an MLP on $10$ randomly projected input features from MNIST with $N=1024$. We also provide the batch size $B = 2N/\epsilon$ at which the temperature $T = 1$. We note that the peak test accuracy is very low, since this task is substantially harder than classifying a complete MNIST image.}
\end{figure}

\begin{figure}[t]
    \centering
        \subfloat[]{\includegraphics[width=0.4\columnwidth]{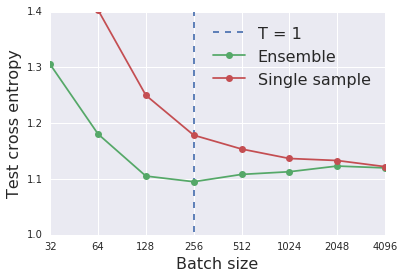}}
        \qquad
        \subfloat[]{\includegraphics[width=0.4\columnwidth]{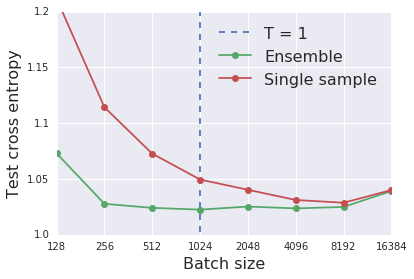}}
        
    \caption{The test cross entropy of NGD as a function of batch size $B$, when training an MLP on $10$ randomly projected input features from MNIST. a) The training set size $N = 4096$. The cross-entropy of the ensemble shows a minimum at $B = 256$, for which $T = \epsilon N/(2N) = 1$. b) The training set size $N = 16384$. As the training set size increases, the minimum near $T \approx 1$ is increasingly flat.}
\end{figure}

In Figure 5a we exhibit the test cross entropy for the same model when we increase the training set size to $N=4096$. Now the cross entropy of the ensemble shows a minimum at an increased batch size of $B = 256$. Once again, this matches the Bayesian prediction. Notice that this leads to a linear scaling rule between the batch size and the training set size ($B \propto N$) as was previously observed for SGD by \cite{smith2017understanding}. In Figure 5b, we further increase the training set size to $N = 16384$. The minimum in the test cross entropy becomes increasingly flat as the training set grows.

\begin{figure}[t]
    \centering
        \subfloat[]{\includegraphics[width=0.37\columnwidth]{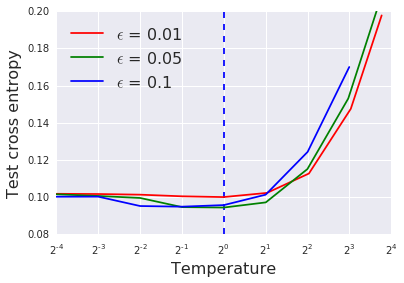}}
        \qquad
        \subfloat[]{\includegraphics[width=0.5\columnwidth]{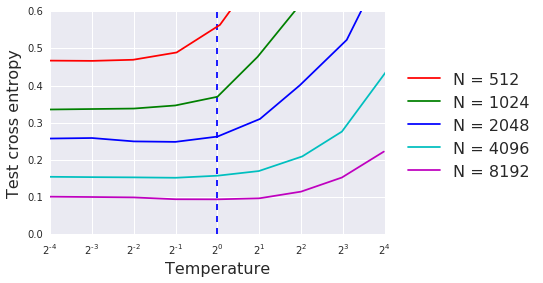}}
        
    \caption{The test cross entropy of NGD as a function of temperature $T = \frac{\epsilon N}{2B} \left(1 - \frac{B}{N} \right)$, when training a CNN on MNIST. a) The training set size $N=8192$, and we plot the test cross entropy as a function of temperature $T$, for a range of learning rates $\epsilon$. b) The learning rate $\epsilon = 0.05$, and we plot the test cross entropy as a function of temperature for a range of training set sizes $N$.}
\end{figure}

\subsection{Stochastic NGD and a CNN}

In Figure 6, we apply NGD to train a CNN on MNIST. Our model is comprised of 3 convolutional layers and 1 fully connected softmax layer. Each convolutional layer has a $(3,3)$ spatial window, 10 output channels and a stride of $2$. This results in a total of $3490$ trainable parameters ($1890$ for the convolutional layers and $1600$ for the final fully connected layer). We introduce L2 regularization and set $\lambda = 40.0/N$. The Fisher information is estimated during training using independent minibatches of 1024 examples, using a single sampled label per example. To reduce variance, we store a moving average of the Fisher over previous updates, and we set the smoothing coefficient of the Fisher moving average, $\alpha = 0.95^{(\epsilon/0.001)}$. The Fisher damping $\delta = 0.1$. To reduce the burn in time, we initialise the weights at the start of training from the final parameters of a single SGD training run on the full MNIST training set. We average our predictions over ensembles sampled from the final $0.2\%$ of gradient updates at the end of training, and we perform $[850/\epsilon]$ gradient updates per training run.

In Figure 6a, we plot the test cross entropy when the training set size, $N = 8192$, for a range of learning rates $\epsilon$. We set the temperature by choosing the batch size $B = \left[N/(\frac{2T}{\epsilon} + 1) \right]$. As expected, the test cross entropy at constant temperature is largely independent of the learning rate, indicating that stochastic NGD obeys a linear scaling rule between batch size and learning rate, $B \propto \epsilon$ when $B \ll N$, as already observed by many authors for SGD \citep{goyal2017accurate, smith2017dont, balles2016coupling}. The test cross entropy for all learning rates rises rapidly once $T > 1$.

In Figure 6b, we exhibit the test cross entropy when $\epsilon = 0.05$ for a range of training set sizes. The test cross entropy rises as the size of the training set falls, but we consistently observe that the test cross entropy begins to increase rapidly once $T \gtrsim 1$. We note however that the minimum does exhibit a weak shift towards smaller temperatures for smaller training sets. This may reflect the breakdown of the approximations behind Theorem 4, since the samples from minibatch NGD will lie increasingly far from the local minimum as the training set size is reduced, and additionally the model predictions will be increasingly far from the true conditional distribution between inputs and labels.

\section{Conclusions}

We prove that, if the model predictions on the training set approach the true conditional distribution between inputs and labels, the stationary distribution of NGD approaches a Bayesian posterior near local minima as the learning rate $\epsilon \rightarrow 0$, at a temperature $T \approx \epsilon N/(2B)$, where $N$ denotes the training set size and $B$ the batch size. To confirm our claims, we demonstrate that samples from NGD at $B = \epsilon N/2$ are close to the Laplace approximation to the Bayesian posterior. We also find that the test cross entropy of ensembles sampled from NGD are minimised when $B \sim \epsilon N/2$. This confirms minibatch noise can improve the generalization performance of NGD, and suggests Bayesian principles may help predict the optimal batch size during training. Furthermore, we propose a novel algorithm, ``stochastic natural gradient descent", which draws parameterization invariant posterior samples with minibatch gradients by introducing the Jeffreys prior to the stationary distribution.

\subsubsection*{Acknowledgements}

We thank Martin Abadi, Matthew Johnson, Matt Hoffman, Roger Grosse, Yasaman Bahri and Alex Botev for helpful feedback, and Yann Ollivier for gently informing us of errors in our original proof.

\bibliography{iclr2018_conference}

\section*{Appendix}

In what follows $f_{ij} = f^{ij} = f_i^j$ and $|g|$ denotes $|\det{g}|$. The stochastic differential equation (SDE),
\begin{equation}
\frac{du}{dt} = f(u) + \eta(t) \label{eqn:example}
\end{equation}
is intended to describe the ordinary differential equation,
\begin{equation}
\frac{du}{dt} = f(u), \label{eqn:ordinary}
\end{equation}
perturbed by an uncorrelated Gaussian noise source $\eta(t)$ with mean $\mathbb{E}(\eta(t)) = 0$ and variance $\mathbb{E}(\eta(t)\eta^\top(t')) = b(u)b^\top(u) \delta(t-t')$. Equivalently, one could choose to define Equation \ref{eqn:example} as,
\begin{equation}
du^i = f^i(u) dt + b^{ij}(u) dW^j, \label{eqn:example2}
\end{equation}
where $dW^j$ is randomly sampled from the unit Gaussian. However it is important to note that, while Equation \ref{eqn:ordinary} uniquely defines the evolution of an ordinary differential equation, Equations \ref{eqn:example} or \ref{eqn:example2} do not uniquely define a stochastic differential equation. To understand this discrepancy, recall that there are many different numerical methods to approximate solutions to Equation \ref{eqn:ordinary} (``Euler'', ``Euler-Heun'', etc.). For an ordinary differential equation, all of these numerical methods converge to the same evolution for $u(t)$ as the step-size $\Delta t \rightarrow 0$. However, different numerical methods may converge to different probability distributions $P(u_t)$ when applied to Equation \ref{eqn:example2}.

Therefore, in order to describe a stochastic process, we must define a stochastic differential equation, in the form of Equation \ref{eqn:example2}, and also specify an ``interpretation'' of that SDE. This interpretation tells us which numerical methods can be used. The two most common interpretations are the \textit{Ito interpretation}, which uses the Euler method, and the \textit{Stratonovich interpretation}, which uses the Euler-Heun method. We will adopt the Ito interpretation, which can be discretized as follows,
\begin{eqnarray}
u_{n+1}^i &=& u_n^i + \epsilon f^i(u_n) + \sqrt{\epsilon} b(u_n)^{ij} dW^j ,\\
dW^j &\sim& \mathcal{N}(0,1).
\end{eqnarray}
For clarity, we have introduced the learning rate $\epsilon = \Delta t$. The evolution of the probability distribution $P(u_t)$ under the Ito interpretation is governed by the following Fokker-Planck equation,
\begin{equation}
\frac{\partial P}{\partial t} = - \frac{\partial}{\partial u^i} (f^i P) + \frac{1}{2} \frac{\partial^2}{\partial u^i \partial u^j} \left( \left(bb^\top\right)_{ij} P \right), \label{eqn:fokker}
\end{equation}
If there is a probability density $P_s(u)$ for which $\frac{\partial P_s}{\partial t} = 0$, then this defines the stationary distribution, and under mild conditions we expect $P_{t\rightarrow \infty}(u) = P_s$. 

In the Ito interpretation, Brownian motion on a Riemannian manifold is defined by the SDE,
\begin{equation}
d\omega^i = -\frac{1}{2} \left( \left(g^{-1} \frac{dC}{d\omega}\right)_i - |g|^{-1/2} \sum_j \frac{\partial}{\partial \omega^j} \left(g_{ij}^{-1} |g|^{1/2} \right) \right) dt + b_{ij}^{-1} dW^j,
\end{equation}
where $g(\omega) = b(\omega)b^\top(\omega)$ is positive definite. Applying the Euler method, we obtain the update rule,
\begin{eqnarray}
\Delta \omega^i &=& -\frac{\epsilon}{2} \left( \left(g^{-1} \frac{dC}{d\omega} \right)_i +  \sum_j \left(g^{-1} \frac{\partial g}{\partial \omega^j} g^{-1} \right)_{ij} - \frac{1}{2} \sum_j g_{ij}^{-1} \mathrm{Tr} \left(g^{-1} \frac{\partial g}{\partial \omega^j} \right) \right) \nonumber \\
&& + \sqrt{\epsilon } b_{ij}^{-1} dW^j.  \label{eqn:semon}
\end{eqnarray}
This is equivalent to Equation \ref{eqn:massive} from the main text at temperature $T=1$, if we redefine $\epsilon \rightarrow 2\epsilon/N$. Note that to obtain Equation \ref{eqn:semon}, we applied the identity,
\begin{equation}
|g|^{-1/2} \partial_j \left( g^{-1} |g|^{1/2} \right) = - g^{-1} \left( \partial_j g \right) g^{-1} + \frac{1}{2} g^{-1} \mathrm{Tr} \left( g^{-1} \partial_j g \right).
\end{equation}
One can confirm by direct substitution into Equation \ref{eqn:fokker} that the stationary distribution of Equation \ref{eqn:semon} satisfies $P_s = e^{-C(\omega)}|g|^{1/2}$. Finally we generalise this result to arbitrary temperatures by setting $C\rightarrow C'/T$ and $\epsilon \rightarrow \epsilon'T$, thus proving Theorem 2\footnote{We note that the earlier derivation of Equation \ref{eqn:semon} by \citet{girolami2011riemann} differs by a factor of 2. They  claim their scheme samples the uniform prior, although we have confirmed empirically this is not the case.}.

\end{document}